\documentclass{article} 
\usepackage{nips13submit_e,times}
\usepackage{natbib}
\usepackage{graphicx,subfigure}
\usepackage{amsmath}
\usepackage{fancyvrb}
\usepackage{amssymb}
\usepackage{wrapfig}
\usepackage{multirow}
\usepackage{hyperref}
\usepackage{algorithm}
\usepackage{algorithmic}
\usepackage{amsmath,amssymb,amsthm}

\newtheorem{theorem}{Theorem}
\newtheorem{lemma}[theorem]{Lemma}

\theoremstyle{definition}

\newcommand{\LL}[0]{f}
\newcommand{\Ts}[0]{\mathcal{T}}
\newcommand{\vv}[0]{\mathbf{v}}

\newcommand{\RR}[0]{\mathbb{R}}

\newcommand{\rnorm}[0]{\right| }
\newcommand{\lnorm}[0]{\left| }
\newcommand{\Dtheta}[0]{\Delta \theta}

\newcommand{\hess}[0]{\mathbf{H}}
\newcommand{\es}[0]{\mathbf{e}}

\newcommand{\example}[0]{\mathbf{x}}

\newcommand{\Fbf}[0]{\mathbf{F}}

\newcommand{\Abf}[0]{\mathbf{A}}

\newcommand{\el}[1]{_{[#1]}}
\def\slantfrac#1#2{\kern.1em^{#1}\kern-.3em/\kern-.1em_{#2}}

\nipsfinalcopy 

\DeclareMathOperator*{\argmin}{arg\,min}

\begin{document}

\title{On the saddle point problem for non-convex optimization}

\author{Razvan Pascanu\\
Universit\'e de Montr\'eal\\
\texttt{r.pascanu@gmail.com} \\
\And
Yann N. Dauphin\\
Universit\'e de Montr\'eal\\
\texttt{dauphiya@iro.umontreal.ca}
\And
Surya Ganguli\\
Stanford University\\
\texttt{sganguli@standford.edu}\\
\And
Yoshua Bengio\\
Universit\'e de Montr\'eal\\
CIFAR Fellow \\
\texttt{yoshua.bengio@umontreal.ca}\\
}

\maketitle
\begin{abstract}
    A central challenge to many fields of science and engineering involves
    minimizing non-convex error functions over continuous, high dimensional
    spaces.  Gradient descent or quasi-Newton methods are almost ubiquitously
    used to perform such minimizations, and it is often thought that a main
    source of difficulty for the ability of these local methods to find the
    global minimum is the proliferation of local minima with much higher error
    than the global minimum.  Here we argue, based on results from statistical
    physics, random matrix theory, and neural network theory, that a deeper and
    more profound difficulty originates from the proliferation of saddle
    points, not local minima, especially in high dimensional problems of
    practical interest. Such saddle points are surrounded by high error
    plateaus that can dramatically slow down learning, and give the illusory
    impression of the existence of a local minimum.  Motivated by these
    arguments, we propose a new algorithm, the saddle-free Newton method, that
    can rapidly escape high dimensional saddle points, unlike gradient descent
    and quasi-Newton methods.  We apply this algorithm to deep neural network
    training, and provide preliminary numerical evidence for its superior
    performance. 
\end{abstract}

\section{Introduction}

It is often the case that our geometric intuition, derived from our experience
within a low dimensional physical world, is woefully inadequate for thinking
about the geometry of typical error surfaces in high-dimensional spaces.
Consider, for example, minimizing a typical, randomly chosen error function of
a single scalar variable.  More precisely, the error function could be a single
draw of a Gaussian process \citep{Rasmussen05}.  Such a random error function
would, with high probability over the choice of function, have many local
minima (maxima), in which the gradient vanishes and the second derivative is
negative (positive).  However, it is highly unlikely to have a saddle point
(see Figure~\ref{fig:different_saddle} (a)), in which the gradient ${\it and}$
the second derivative vanish.  Indeed, such saddle points, being a degenerate
condition,  would occur with probability zero. Similarly, typical, random error
functions on higher dimensional spaces of $N$ variables are likely to have many
local minima for very small $N$.  However,  as we review below, as the
dimensionality $N$ increases, local minima with high error relative to the
global minimum occur with a probability that is exponentially small in $N$. 

In general, consider an error function $\LL(\theta)$  where $\theta$ is an $N$
dimensional continuous variable. A critical point is by definition a point
$\theta$ where the gradient of $\LL(\theta)$ vanishes.  All critical points of
$f(\theta)$ can be further characterized by the curvature of the function in
its vicinity, as described by the eigenvalues of the Hessian.  Note that the
Hessian is symmetric and hence the eigenvalues are real numbers. The following
are the four possible scenarios: 

\begin{itemize}
\item If all eigenvalues are non-zero and  positive, then the critical point is
a local minimum.  
\item If all eigenvalues are non-zero and negative, then the critical point is
a local maximum.
\item If the eigenvalues are non-zero and we have both positive and negative
eigenvalues, then the critical point is a saddle point with a \emph{min-max}
structure (see Figure~\ref{fig:different_saddle} (b)). That is, if we restrict
the function $\LL$ to the subspace spanned by the eigenvectors corresponding to
positive (negative) eigenvalues, then the saddle point is a maximum (minimum)
of this restriction.
\item If the Hessian matrix is singular, then the \emph{degenerate} critical
point can be a saddle point, as it is, for example, for $\theta^3, \theta
\in\RR$ or for the monkey saddle (Figure~\ref{fig:different_saddle} (a) and
(c)).  If it is a saddle, then, if we restrict $\theta$ to only change along
the direction of singularity, the restricted function does not exhibit a
minimum nor a maximum; it exhibits, to second order, a plateau. When moving from
one side to other of the plateau, the eigenvalue corresponding to this picked
direction generically changes sign, being exactly zero at the critical point.
Note that an eigenvalue of zero can also indicate the presence of a gutter
structure, a degenerate minimum, maximum or saddle, where a set of connected
points are all minimum, maximum or saddle structures of the same shape
and error.  In Figure~\ref{fig:different_saddle} (d) it is shaped as a circle.
The error function looks like the bottom of a wine bottle, where all points
along this circle are minimum of equal value.
\end{itemize}

\begin{figure}[t]
    \centering
    \begin{minipage}{0.24\textwidth}
        \centering
        \includegraphics[width=0.85\columnwidth, clip=true, trim=2cm .5cm 2cm .5cm]{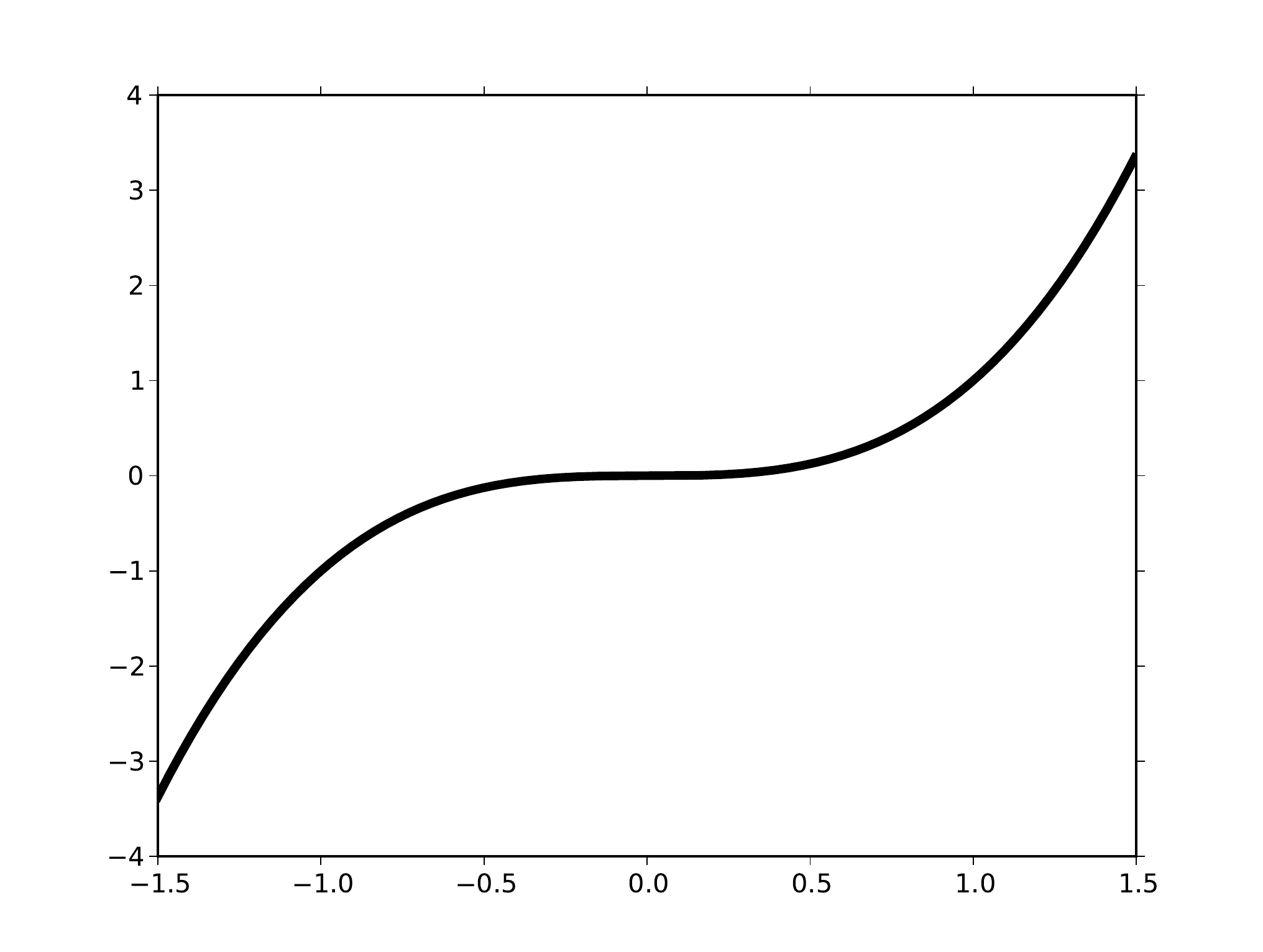}
    \end{minipage}
    \hfill
    \begin{minipage}{0.24\textwidth}
        \centering
        \includegraphics[width=1.\columnwidth, clip=true, trim=2cm 1cm 2cm 1cm]{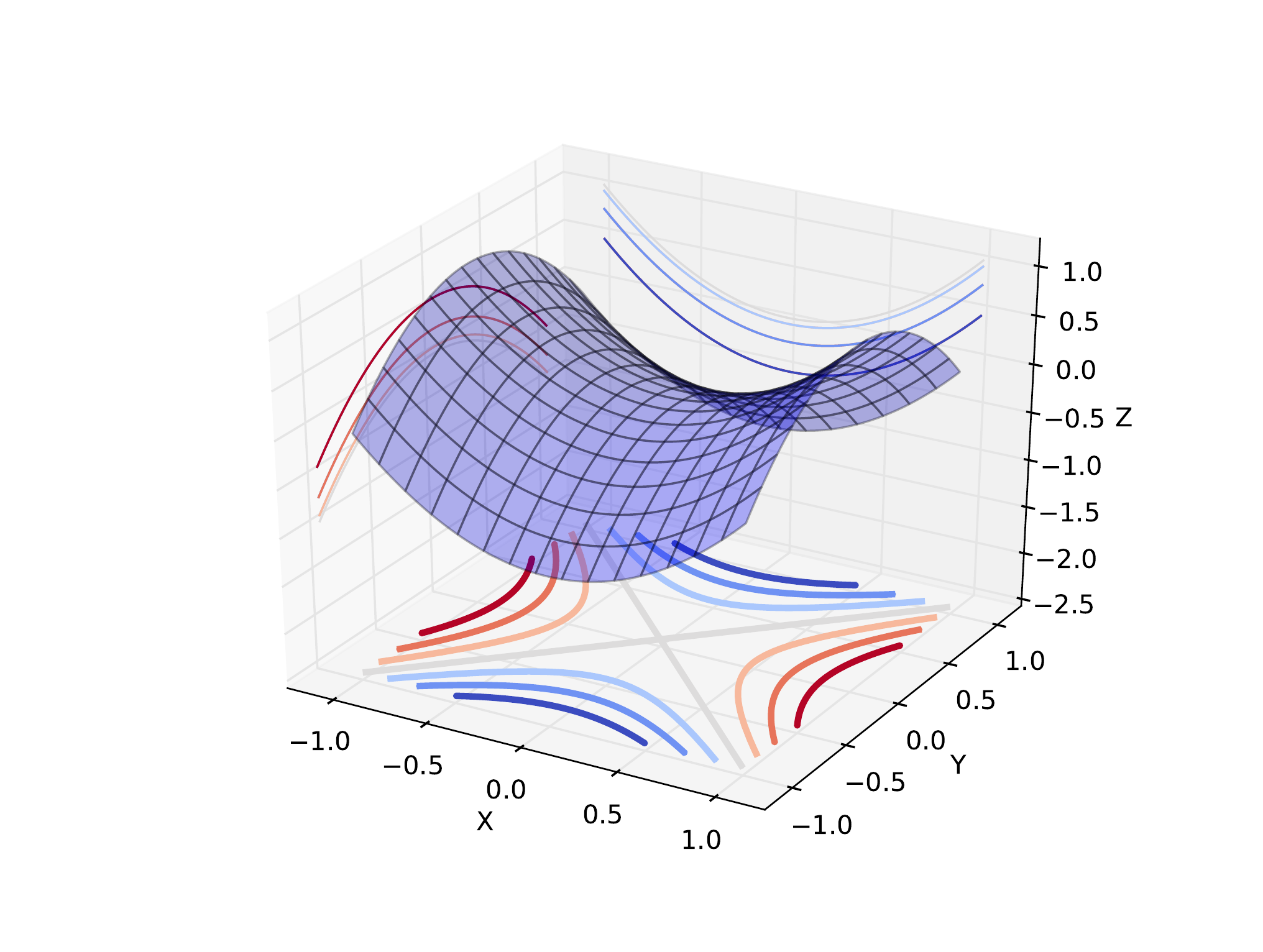}
    \end{minipage}
    \begin{minipage}{0.24\textwidth}
        \centering
        \includegraphics[width=1.\columnwidth, clip=true, trim=2cm 1cm 2cm 1cm]{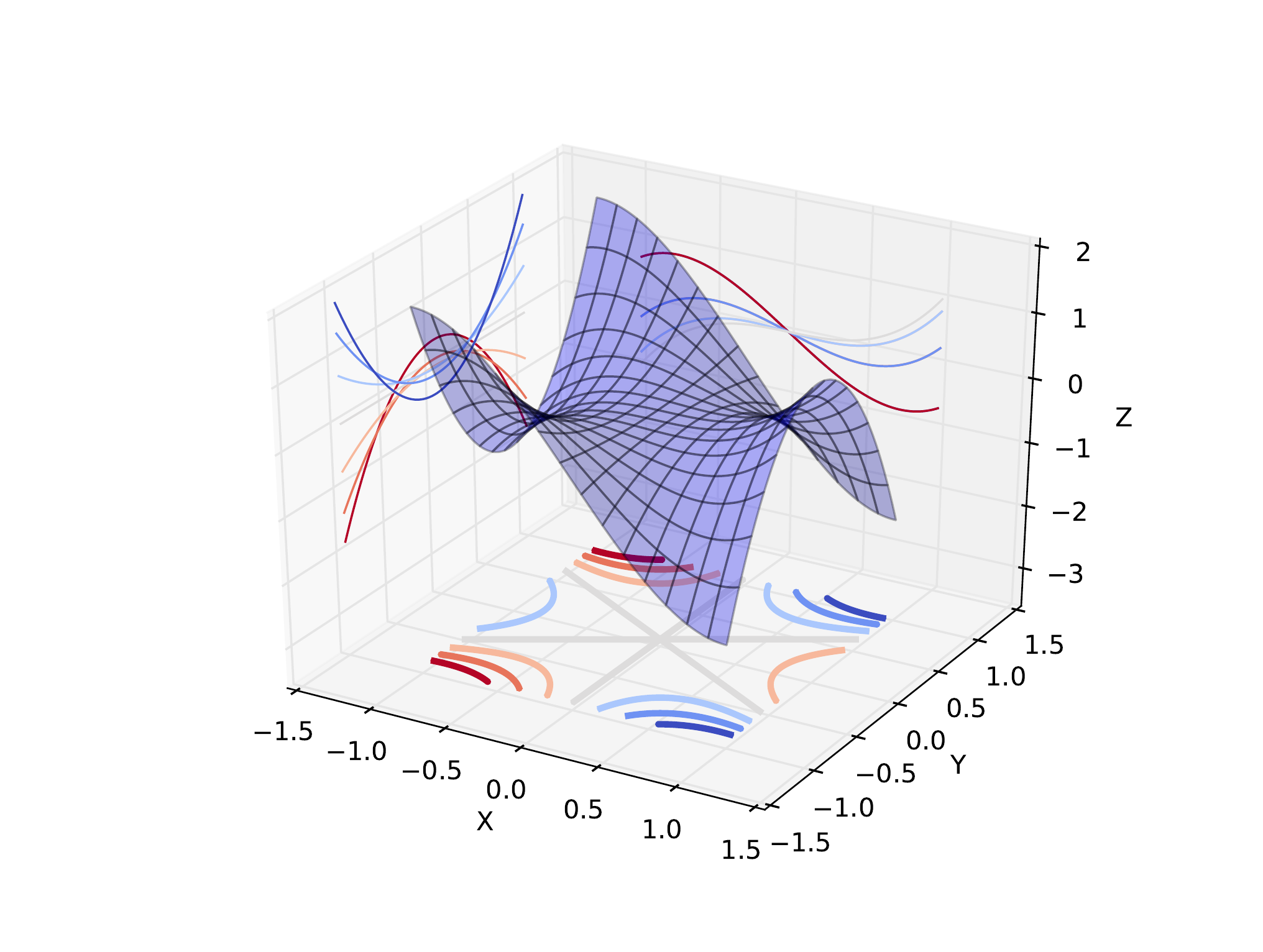}
    \end{minipage}
    \hfill
    \begin{minipage}{0.24\textwidth}
        \centering
        \includegraphics[width=1.\columnwidth, clip=true, trim=2cm 2cm 2cm 2cm]{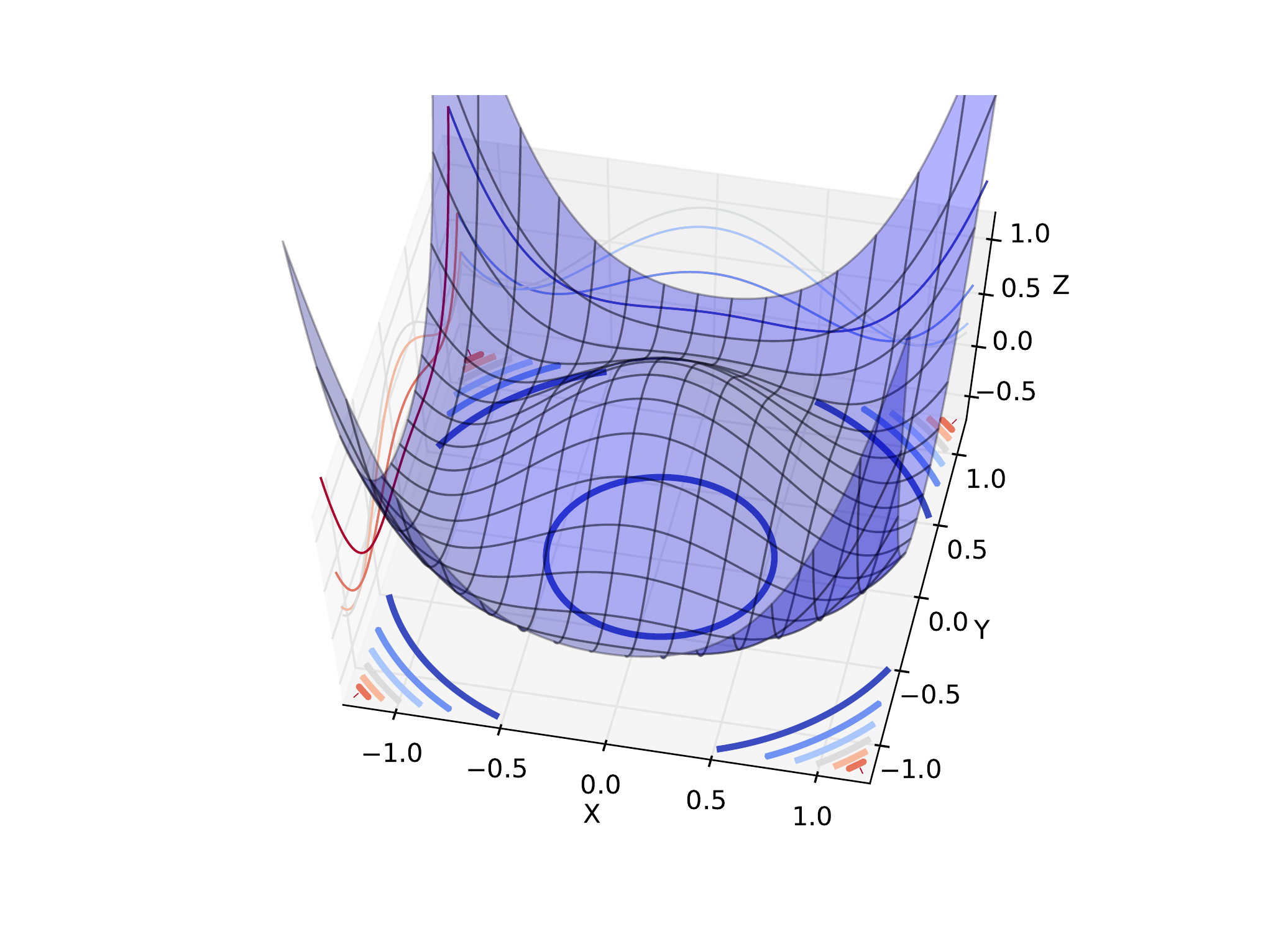}
    \end{minipage}
    \\
    \vspace{3mm}
    \begin{minipage}{0.24\textwidth}
        \centering
        (a) 
    \end{minipage}
    \hfill
    \begin{minipage}{0.24\textwidth}
        \centering
        (b) 
    \end{minipage}
    %
    \begin{minipage}{0.24\textwidth}
        \centering
        (c) 
    \end{minipage}
    \hfill
    \begin{minipage}{0.24\textwidth}
        \centering
        (d) 
    \end{minipage}
\caption[Different saddle point structures] {Illustrations of three different
types of saddle points (a-c) plus a gutter structure (d).  Note that for the
gutter structure, any point from the circle $x^2 + y^2 = 1$ is a minimum.  The
shape of the function is that of the bottom of a bottle of wine. This means
that the minimum is a ``ring'' instead of a single point. The Hessian is
singular at any of these points. (c) shows a Monkey saddle where you have both
a min-max structure as in (b) but also a 0 eigenvalue, which results, along
some direction, in a shape similar to (a).
}
\label{fig:different_saddle}
\end{figure}

A plateau is an almost flat region in some direction.  This structure is given
by having the eigenvalues (which describe the curvature) corresponding to the
directions of the plateau be \emph{close to 0}, but \emph{not exactly 0}. Or,
additionally, by having a large discrepancy between the norm of the
eigenvalues. This large difference would make the direction of ``relative''
small eigenvalues look like flat compared to the direction of large
eigenvalues.

Saddle points are unstable under gradient descent dynamics on the error
surface, as the dynamics is repelled away from the saddle by directions of
negative curvature. However, this repulsion can occur slowly due to plateaus of
small negative eigenvalues.  A similar slow-down can occur for local minima
with small positive eigenvalues. Second order methods were designed to rapidly
descend local minima by rescaling gradient steps by the inverse eigenvalues.
However, the Newton method does not treat saddle points appropriately; as
argued below, saddle-points instead become attractive under the Newton
dynamics.

Thus given the proliferation of saddle points, not local minima, in high
dimensional problems, the entire theoretical justification for quasi-Newton
methods, i.e. the ability to rapidly descend to the bottom of a convex local
minimum, becomes less relevant in high dimensional nonconvex optimization.  In
this work, we propose an algorithm whose theoretical justification is motivated
by ensuring rapid escape from saddle points.  This algorithm leverages second
order-curvature information in a fundamentally different way than quasi-Newton
methods, and also, in preliminary results, out-performs them in some high
dimensional problems.

\subsection{The prevalence of saddle points}

Here we review arguments from disparate literatures suggesting that saddle
points, not local minima, provide a fundamental impediment to rapid high
dimensional non-convex optimization.  One line of evidence comes from
statistical physics, where the nature of critical points of random Gaussian
error functions on high dimensional continuous domains is studied
\citet{Bray07, Fyodorov07} using replica theory,  a theoretical technique for
analyzing high dimensional systems with \emph{quenched disorder} like spin
glasses (see \citet{Parisi07} for a recent review).  In particular,
\citet{Bray07} counted the typical number of critical points of a random
function in a finite cubic volume in $N$ dimensional space within a range of
error $\epsilon$ and index $\alpha$.  By definition the index $\alpha$ is the
fraction of negative eigenvalues of the Hessian at the critical point.  Of
course every such function has a unique global minimum at $\epsilon =
\epsilon_{\text{min}}$ and $\alpha = 0$ and a global maximum at $\epsilon =
\epsilon_{\text{max}}$ and $\alpha = 1$, where $\epsilon_{\text{min}}$  and
$\epsilon_{\text{max}}$ do not depend strongly on the detailed realization of
the random function due to concentration of measure.  In \citet{Bray07}, the
authors found that any such function, over a large enough volume, has
exponentially many critical points, but in the $\epsilon-\alpha$ plane all the
critical points are overwhelmingly likely to be located on a monotonically
increasing curve $\epsilon^*(\alpha)$ that rises from $\epsilon_{\text{min}}$
to $\epsilon_{\text{max}}$ as $\alpha$ ranges from $0$ to $1$.   Indeed the
probability that a critical point lies an $O(1)$ distance off this curve, both
over the choice of a critical point for a given function and over the choice of
random function from the Gaussian ensemble, is exponentially small in the
dimensionality $N$, for large $N$. 

Intuitively, these theoretical results indicate that critical points at any
intermediate error $\epsilon$ above the global minimum $\epsilon_{\text{min}}$
are exponentially likely to be saddle points, with the fraction of negative
eigenvalues $\alpha$ of the Hessian monotonically increasing with $\epsilon$.
Furthermore, any critical point with a very small fraction of negative
eigenvalues is exponentially likely to occur at low error $\epsilon$ close to
$\epsilon_{\text{min}}$.  In particular, any local minimum with $\alpha=0$ will
have an error exceedingly close to that of the global minimum.  Thus the
optimization landscape of a random Gaussian error function has no local minima
with high error, but is instead riddled with exponentially many saddle points
at errors far above that of the global minimum error.

These results can be further understood via random matrix theory.  Indeed, for
a random error function in $N$ dimensions, the Hessian at a critical point at
error $\epsilon$ can be thought of as an $N$ by $N$ random symmetric matrix
whose eigenvalue distribution depends on $\epsilon$.  \citet{Bray07} found that
the entire eigenvalue distribution of the Hessian took the form of Wigner's
famous semi-circular law \citep{Wigner58}, but shifted by an amount determined
by $\epsilon$.  Indeed, a completely unconstrained random symmetric matrix has
a symmetric eigenvalue density on the real axis shaped like a semicircle with
both mode and mean at $0$.  Thus any eigenvalue is positive or negative with
probability $1/2$.  The eigenvalue distribution of the Hessian of the critical
point at error $\epsilon$ is a shifted semicircle, where the shift ensures that
the fraction of negative eigenvalues $\alpha$ is given exactly by the solution
to $\epsilon = \epsilon^*(\alpha)$.  When the error $\epsilon =
\epsilon_\text{min}$, the semicircular distribution of the Hessian is shifted
so far to the right that all eigenvalues are positive, corresponding to the
global minimum.  As the error $\epsilon$ of the critical point increases, the
semi-circular eigenvalue distribution shifts to the left, and the fraction of
negative eigenvalues $\alpha$ increases.  At intermediate error $\epsilon$,
half way between $\epsilon_\text{min}$ and  $\epsilon_\text{max}$, the
semicircular distribution of eigenvalues has its mode at $0$.  This implies
that the highest density of eigenvalues occurs near $0$, and so a typical
critical point at intermediate error not only has many negative curvature
directions, but also many approximate plateau directions, in which a finite
fraction of eigenvalues of the Hessian lie near $0$.   The existence of these
approximate plateau directions, in addition to the negative directions, would
of course have significant implications for high dimensional non-convex
optimization, in particular, dramatically slowing down gradient descent
dynamics.  Moreover, \citet{Fyodorov07} give a review of work in which
qualitatively similar results are derived for random error functions
superimposed on a quadratic error surface.   

Before continuing, we note that  the random matrix perspective concisely and
intuitively crystalizes the striking difference between the geometry of low and
high dimensional error surfaces.  For $N=1$, the Hessian of a random function
is a single random number, and so with overwhelming probability it will be
positive or negative; the event that it is $0$ is a set of measure zero.  This
reflects the intuition that saddle points in $1$ dimension are extremely
unlikely, while maxima and minima always occur.  Alternatively, an
unconstrained random Gaussian matrix in $N$ dimensions has a probability
$O(e^{-N})$ that all of its eigenvalues are positive.  This fact reflects that
local minima with error far higher than the global minima are exponentially
unlikely.  The Hessian at a critical point with error very close to the global
minimum is not a fully unconstrained random Gaussian matrix; the fact that the
error is so small, shifts its eigenvalue distribution to the right, so that
more eigenvalues are positive \citet{Bray07, Fyodorov07}.

Thus the above work indicates that for typical, generic functions chosen from a
random Gaussian ensemble of functions, local minima with high error are
exponentially rare in the dimensionality of the problem, but saddle points with
many negative and approximate plateau directions are exponentially likely at
high error.  However, are the error landscapes of practical problems of
interest somehow not reflective of generic error landscapes, and therefore not
riddled with similar saddle points?  Although our proposed algorithm described
below should be generically useful for a wide variety of problems, given that
our proximal motivation is ultimately training deep neuronal networks, we
review evidence from that literature that saddle points also play a prominent
role in the learning dynamics. 

In \citet{Baldi89} the error surface of a single hidden layer MLP with linear
units is analysed.  The number of hidden units is assumed to be less than the
number of inputs units. Such an error surface shows only saddle-points and
\emph{no} local minimum or local maximum. This result is qualitatively
consistent with the observation made by \citet{Bray07}. In fact, as long as we
do not \emph{get stuck} in the plateaus surrounding these saddle points, for
such a model we are guaranteed to obtain the global minimum of the error.
Indeed \citet{Saxe-ICLR2014}, analyzed the dynamics of learning in the presence
of these saddle points, and showed that they arise due to scaling symmetries in
the weight space of deep linear feedforward models.  These scaling symmetries
enabled \citet{Saxe-ICLR2014} to find new exact solutions to the nonlinear
dynamics of learning in deep linear networks. These learning dynamics exhibit
plateaus of high error followed by abrupt transitions to better performance,
and they qualitatively recapitulate aspects of the hierarchical development of
semantic concepts in infants \citep{Saxe-Cogsci}.

In \citet{Saad95} the dynamics of stochastic gradient descent are analyzed for
soft committee machines.  The paper explores how well a student model can learn
to imitate a teacher model which was randomly sampled.  An important
observation of this work is showing that learning goes through an initial phase
of \emph{being trapped in the symmetric subspace}. In other words, due to
symmetries in the randomly initialized weights, the network has to traverse one
or more plateaus that are caused by units with similar behaviour.
\citet{MagnusSA_98, Inoue03} provides further analysis, stating that the
initial phase of learning is plagued with saddle point structures caused by
symmetries in the weights.  Intuitively, the escape from these saddle points
corresponds to weight differentiation of the afferent synapses onto hidden
units.  Being trapped along the symmetric subspace corresponds to pairs of
hidden units computing the same function on the input distribution.  Exiting
the symmetric subspace corresponds to hidden units learning to become different
from each other, thereby specializing and learning the internal representations
of the teacher neural network.  Interestingly, the error function in the
vicinity of the symmetric subspace has a saddle point structure, and signals
that hidden unit differentiation will lower error by following directions of
negative curvature.  Thus first order gradient descent dynamics yields a
plateau in the error because it is difficult for such dynamics to rapidly
escape from the saddle point in the vicinity of the symmetric subspace by
specializing hidden units.    

\citet{Mizutani10} looks at the effect of negative curvature for learning and
implicitly at the effect of saddle point structures in the error surface. Their
findings are similar. A proof is given where the error surface of a single
layer MLP is shown to have saddle points (where the Hessian matrix is
indefinite).
 
\section{Optimization algorithms near saddle points}

The above work suggests that both in typical randomly chosen high dimensional
error surfaces, and neural network training error surfaces, a proliferation of
saddle points with error much higher than the local minimum constitute the
predominant obstruction to rapid non convex optimization.  We now provide a
theoretically justified algorithm to deal with this problem. We will focus on
nondegenerate saddle points, namely those for which the Hessian is not exactly
singular. These critical points can be locally analyzed based on a unique
reparametrization of the function as described by Morse's lemma (see chapter
7.3, Theorem 7.16 in \citet{callahan2010advanced}). 

This reparametrization is given by taking a Taylor expansion of the function
$\LL$ around the critical point. If we assume that the Hessian is not singular,
then there is a neighbourhood around this critical point where this
approximation is reliable and, since the first order derivatives vanish, the
Taylor expansion is given by:

\begin{equation}
\LL(\theta^* + \Dtheta) = \LL(\theta^*) + \frac{1}{2} (\Dtheta) ^T
\hess \Dtheta
\end{equation}

Let us denote by ${\es\el 1}, \ldots, {\es\el {n_\theta}}$ the eigenvectors of
the Hessian $\hess$ and by ${\lambda\el 1}, \ldots, {\lambda\el {n_\theta}}$
the corresponding eigenvalues.  We can now make a change of coordinates into
the space span by these eigenvectors:

\begin{equation}
\Delta \vv = \frac{1}{2}\left[
\begin{array}{c}
{\es\el{1}}^T \\
\ldots \\
{\es\el{n_\theta}}^T 
\end{array}
\right] \Dtheta
\end{equation}

\begin{equation}
\label{eq:new_system_coord}
\LL(\theta^* + \Dtheta) = \LL(\theta^*) + \frac{1}{2}\sum_{i=1}^{n_\theta}
{\lambda\el i} ({\es\el i}^T \Dtheta)^2 = \LL(\theta^*) + \sum_{i=1}^{n_\theta}
{\lambda\el i} \Delta \vv_i^2 
\end{equation}

For \emph{gradient descent} we can see that, as expected, the gradient points
in the right direction close to a saddle point. Namely, if some eigenvalue
${\lambda\el i}$ is positive, then we move towards $\theta^*$ in that direction
because the restriction of $\LL$ to the corresponding eigenvector is
$\LL(\theta^*) + \lambda\el i \Delta\vv_i^2$, which has a minimum when $\vv_i =
0$. On the other hand, if the eigenvalue is negative, then the gradient descent
will move away from $\theta^*$ which is a maximum when restricting the loss
function to the corresponding eigenvector of said eigenvalue. 

The downside of gradient descent is not the direction, but the \emph{size} of
the step along each eigenvector. The step we will take, for any direction
$\es\el i$, is given by $-2{\lambda\el i} \Delta \vv_i$.  Because the gradients
are proportional to the corresponding eigenvalues of the Hessian, the
eigenvalue dictates how fast we move in each direction. Note that also, to
avoid divergence, the learning rate has to be at most
$\slantfrac{1}{|\lambda\el{max}|}$.  Therefore, if there is a large discrepancy
between eigenvalues, then gradient descent will have to take very small steps
in some directions. This means that it might takes  a very long time to move
away form the critical point, if the critical point is a saddle point, or to
the critical point if it is a minimum. 

The \emph{Newton method} solves the slowness problem by properly rescaling the
gradients in each direction with the inverse of the corresponding eigenvalue.
The step we take now is given by $-\Delta \vv_i$.  However, this approach can
result in moving in the wrong direction. Specifically, if an eigenvalue is
negative, then by multiplying with its inverse, the Newton method would change
the sign of the gradient along this eigenvector. Instead of taking the step
away from $\theta^*$ in the direction of negative curvature (where $\theta^*$
is a maximum), Newton method moves towards $\theta^*$. This effectively makes
$\theta^*$ an \emph{attractor} for the dynamics of the Newton method, making
this algorithm converge to this unstable critical point.  Therefore, while
gradient descent might still escape saddle points in finite time, Newton method
can not and it will converge to said saddle point. 

A \emph{trust region} approach is a practical implementation of second order
methods for non-convex problems, where the Hessian is damped to remove negative
curvature.  Damping the Hessian by adding the identity matrix times some
constant $\alpha$ is equivalent to adding $\alpha$ to each of the eigenvalues
of the Hessian.  That is, we now rescale the gradients by multiplying them with
$\slantfrac{1}{{\lambda\el i} + \alpha}$, resulting in the step $-
\left(\slantfrac{{\lambda\el i}}{{\lambda\el i} + \alpha}\right) \Delta \vv_i$.
In particular, to deal with negative curvature, one has to increase the damping
coefficient $\alpha$ enough such that even for the most negative eigenvalue
$\lambda\el{min}$ we have $\lambda\el{min} + \alpha > 0$. This ensures moving
in a descent direction.  The downside is again the step size in each direction.
Adding a fixed constant to each eigenvalue makes the ratio
$\slantfrac{{\lambda\el i}}{{\lambda\el i} + \alpha}$ far from 1 for most
eigenvalues, especially when we have a large discrepancy between them. 

Beside damping, another approach of dealing with negative curvature for second
order methods is to ignore them. This can be done regardless of the
approximation of the Newton method used, for example as either a truncated
Newton method or a BFGS approximation (see \citet{NumOptBook} chapters 4 and
7).  By ignoring direction of negative curvature, we will not be able to escape
saddle points, as there is no direction in which we move away from $\theta^*$.
Damping and ignoring the directions of negative curvature are the main existing
approaches to deal with negative curvature. 

Natural gradient descent is a first order method that relies on the curvature
of the parameter manifold. That is, we take a step that induces a constant
change in the behaviour of the model as measured by the KL-divergence between
the model before taking the step and after. For example, the recently proposed
Hessian-Free Optimization \citep{Martens10} was shown to be a variant of
natural gradient descent \citep{Pascanu+Bengio-ICLR2014}.  The algorithm ends
up doing an update similar to the Newton method, just that instead of inverting
the Hessian we invert Fisher Information Matrix, $\Fbf$, which is positive
definite by construction.  It is argued in \citet{MagnusSA_98, Inoue03} that
natural gradient descent can address certain saddle point structures
effectively.  Specifically, it can resolve those saddle points arising from
having units behaving very similar.  In \citet{Mizutani10}, however, it is
argued that natural gradient descent does also suffer when negative curvature
is present. One particular known issue is the over-realizable regime, when the
model is over complete.  In this situations, while there exists a stationary
solution $\theta^*$, the Fisher matrix around $\theta^*$ is rank deficient.
Numerically, this means that the Gauss Newton direction can be (close to)
orthogonal to the gradient at some distant point from $\theta^*$
\citep{Mizutani10}. A line search in this direction would fail and lead to the
algorithm converging to some non-stationary point.  Another weakness of natural
gradient is that the \emph{residual} $\mathbf{S}$ defined as the difference
between the Hessian and the Fisher Information Matrix can be large close to
certain saddle points that exhibit strong negative curvature.  This means that
the landscape close to these critical points can be dominated by $\mathbf{S}$,
meaning that the rescaling provided by $\Fbf^{-1}$ is not optimal in all
directions as it does not match the eigenvalues of the Hessian.

The same is true for TONGA \citet{RouxMB07}, an algorithm similar to natural
gradient descent. TONGA relies on the covariance of the gradients for the
rescaling factor. As these gradients vanish close to a critical point, their
covariance will indicate that one needs to take much larger steps then needed
close to critical points. 

\section{Generalized trust region methods}

We will refer to a straight forward extension of trust region methods as 
\emph{generalized trust region methods}. The extension involves two simple 
changes of the method. The first one is that we allow to take a first order 
Taylor expansion of the function to minimize instead of always relying on a 
second order Taylor expansion as typically done in trust region methods. 

The second change is that we replace the constraint on the norm of the step
$\Dtheta$ by a constraint on the distance between $\theta$ and $\theta +
\Dtheta$. The distance measure is also not specified and can be 
specific to the instance of generalized trust region method used.  If we define
$\Ts_k(\LL, \theta, \Dtheta)$ to indicate a $k$-th order Taylor series
expansion of $\LL$ around $\theta$ evaluated at $\theta + \Dtheta$, then we can
summarize a generalized trust region as:

\begin{equation}
    \begin{aligned}
        \Delta \theta & = 
            \argmin_{\Delta \theta} 
                    \Ts_k\{\LL, \theta, \Delta\theta\} &
                    \text{ with } k \in \{1, 2\} \\
        & \text{s. t. } d(\theta, \theta + \Delta \theta) \leq \Delta & 
    \end{aligned} 
\end{equation}

\section{Addressing the saddle point problem}
\label{sec:addressing}

In order to address the saddle point problem, we will look for a solution
within the family of generalized trust region methods. We know that using the
Hessian as a rescaling factor can result in a non-descent direction because of
the negative curvature. The analysis above also suggest that correcting
negative curvature by an additive term results in a suboptimal step, therefore
we want the resulting step from this trust region method to not be a function
of the Hessian. We therefore use a first order Taylor expansion of the loss.
This means that the curvature information has to come from the constraint by
picking a suitable distance measure $d$.

\subsection{Limit the influence of second order terms -- saddle-free Newton Method}
\label{sec:sfnm}

The analysis carried out for different optimization techniques states that,
close to nondegenerate critical points, what we want to do is to rescale the
gradient in each direction $\es\el i$ by $\slantfrac{1}{|\lambda\el i|}$.  This
achieves the same optimal rescaling as the Newton method, while preserving the
sign of the gradient and therefore avoids making saddle point attractors of the
dynamics of learning. The idea of taking the absolute value of the eigenvalues
of the Hessian was suggested before. See, for example, in \citet[chapter
3.4]{NumOptBook} or in \citet[chapter 4.1]{Murray10}. However, we \emph{are not
aware} of any proper justification of this algorithm or even a proper detailed
exploration (empirical or otherwise) of this idea. 

The problem is that one can not simply replace $\hess$ by $|\hess|$, where
$|\hess|$ is the matrix obtained by taking the absolute value of each
eigenvalue of $\hess$, without proper justification. For example, one obvious
question is: are we still optimizing the same function?  While we might be
able to argue that we do the right thing close to critical points, can we do
the same far away from these critical points? In what follows we will provide
such a justification.

Let us consider the function we want to minimize $\LL$ by employing a
generalized trust region method that works on a first order approximation of
$\LL$ and enforces some constraint on the step taken based on some distance
measure $d$ between $\theta$ and $\theta + \Delta \theta$. Since the minimum of
the first order approximation of $\LL$ is at infinity, we know that within this
generalized trust region approach we will always jump to the border of the
trust region. 

So the proper question to ask is how far from $\theta$ can we trust our first
order approximation of $\LL$. One measure of this trustfulness is given by how
much the second order term of the Taylor expansion of $\LL$ influences the
value of the function at some point $\theta + \Delta \theta$. That is we want
the following constraint to hold:

\begin{equation}
\label{eq:constr}
d(\theta, \theta + \Delta \theta) = \lnorm \LL(\theta) + \nabla \LL \Dtheta + 
\frac{1}{2}\Dtheta^T \hess \Dtheta - \LL(\theta) - \nabla \LL \Dtheta \rnorm =
\frac{1}{2}\lnorm \Dtheta^T \hess \Dtheta\rnorm \leq \Delta
\end{equation}

where $\nabla \LL$ is the partial derivative of $\LL$ with respect to $\theta$
and $\Delta \in \RR$ is some some small value that indicates how much
discrepancy we are willing to accept between our first order approximation of
$\LL$ and the second order approximation of $\LL$.Note that the distance
measure $d$ takes into account the curvature of the function. 

Equation \eqref{eq:constr} is also not easy to solve for $\Dtheta$ in more than
one dimension. If we resolve the absolute value by taking the square of the
distance we get a function that is quartic in $\Dtheta$ (the term is raised to
the power 4). We address this problem by relying on the following Lemma. 

\begin{lemma}
\label{lemma:constraint}
Let $\Abf$ be a nonsingular square matrix in $\RR^{n} \times \RR^{n}$, and
$\example \in \RR^n$ be some vector. Then it holds that $|\example^T \Abf
\example | \leq \example^T |\Abf| \example$, where $|\Abf|$ is the matrix
obtained by taking the absolute value of each of the eigenvalues of $\Abf$.
\end{lemma}

\begin{proof}
Let $\es\el 1, \ldots \es\el n$ be the different eigenvectors of $\Abf$ and
$\lambda\el 1, \ldots \lambda\el n$ the corresponding eigenvalues. We now
re-write the identity by expressing the vector $\example$ in terms of these
eigenvalues: 

\begin{align}
|\example^T \Abf \example|& = \left|\sum_i (\example^T \es\el i){\es\el{i}}^T \Abf \example \right| \nonumber \\
& = \left|\sum_i (\example^T \es\el i)\lambda\el i ({\es\el  i}^T \example) \right| \nonumber \\
& = \left|\sum_i \lambda\el i (\example^T \es\el i)^2 \right| \nonumber
\end{align}

We can now use the triangle inequality $|\sum_i x_i| \leq \sum_i |x_i|$ and get
that 

\begin{align}
|\example^T \Abf \example|& \leq \sum_i |(\example^T \es\el i)^2 \lambda\el i| \nonumber \\
& = \sum_i (\example^T \es\el i)|\lambda\el i| ({\es\el i}^T\example) \nonumber \\
& = \example^T |\Abf | \example \nonumber
\end{align}

\end{proof}

Instead of using the originally proposed distance measure, based on
lemma~\ref{lemma:constraint}, we will approximate the distance by its upper
bound given by $\Dtheta |\hess| \Dtheta$, resulting in the following
generalized trust region method:

\begin{equation}
    \begin{aligned}
        \Delta \theta & = 
            \argmin_{\Delta \theta} 
                    \LL(\theta) + \nabla \LL \Delta \theta \\
        & \text{s. t. } \Delta \theta^T |\hess|\Delta \theta \leq \Delta 
    \end{aligned} 
\end{equation}

Note that as discussed before, the inequality constraint can be turned into a
equality one as the first order approximation of $\LL$ has a minimum at
infinity and we always jump to the border of the trust region.  Similar to
\citet{Pascanu+Bengio-ICLR2014}, we can use Lagrange multipliers to get the
solution of this constraint optimization. This gives (up to a scalar that we
fold into the learning rate) a step of the form:

\begin{equation}
\Delta \theta = -\nabla \LL |\hess|^{-1}
\end{equation}

The algorithm is a trust region method that uses the curvature of the function
to define the shape of the trust region. It allows to move further in
directions of low curvature and enforces to move less in direction of high
curvature. If the Hessian is positive definite the method behaves identically
to the Newton method. Close to a nondegenerate critical points, it takes the
optimum step, by scaling based on the eigenvalues of the Hessian which describe
the geometry of the surface locally, while moving away from the critical point
in the direction of negative curvature. 

\section{Experimental results -- empirical evidence for the saddle point
problem}

In this section we run experiments on a scaled down version of MNIST, where
each input image is rescaled to be of size $10\times 10$. This rescaling allows
us to construct models that are small enough such that we can implement the
exact Newton method and saddle-free Newton method, without relying on any kind
of approximations. 

As a baseline we also consider minibatch stochastic gradient descent, the de
facto optimization algorithm for such models. We additionally use momentum to
improve the convergence speed of this algorithm. The hyper-parameters of
minibatch SGD -- the learning rate, minibatch size and the momentum constant --
are chosen using random search \citep{Bergstra+Bengio-2012-small}. We pick the
best configurations from approximately 80 different choices. The learning rate
and momentum constant are sampled on a logarithmic scale, while the minibatch
size is sampled from the set $\{1, 10, 100\}$.  The best performing
hyper-parameter values for SGD are provided in Table~\ref{table:hyper_sgd}.

Damped Newton method is a trust region method where we damp the Hessian $\hess$
by adding the identity times the damping factor. No approximations are used in
computing the Hessian or its inverse (beside numerical inaccuracy due to
machine precision).  For the saddle-free Newton we also damp the matrix
$|\hess|$, obtained by taking the absolute value of the eigenvalues of the
Hessian.  At each step, for both methods, we pick the best damping coefficient
among the following values: $\{10^0, 10^{-1}, 10^{-2}, 10^{-3}, 10^{-4},
10^{-5} \}$.  We do not perform an additional line search for the step size,
but rather consider a fixed learning rate of 1.  Note that by searching over
the damping coefficient we implicitly search for the optimum step size as well.
These two methods are run in batch mode.

\begin{figure}[t]
    \centering
    \begin{minipage}{0.49\textwidth}
        \centering
        \includegraphics[width=1.\columnwidth, clip=true, trim=3cm 0cm 6cm 0cm]{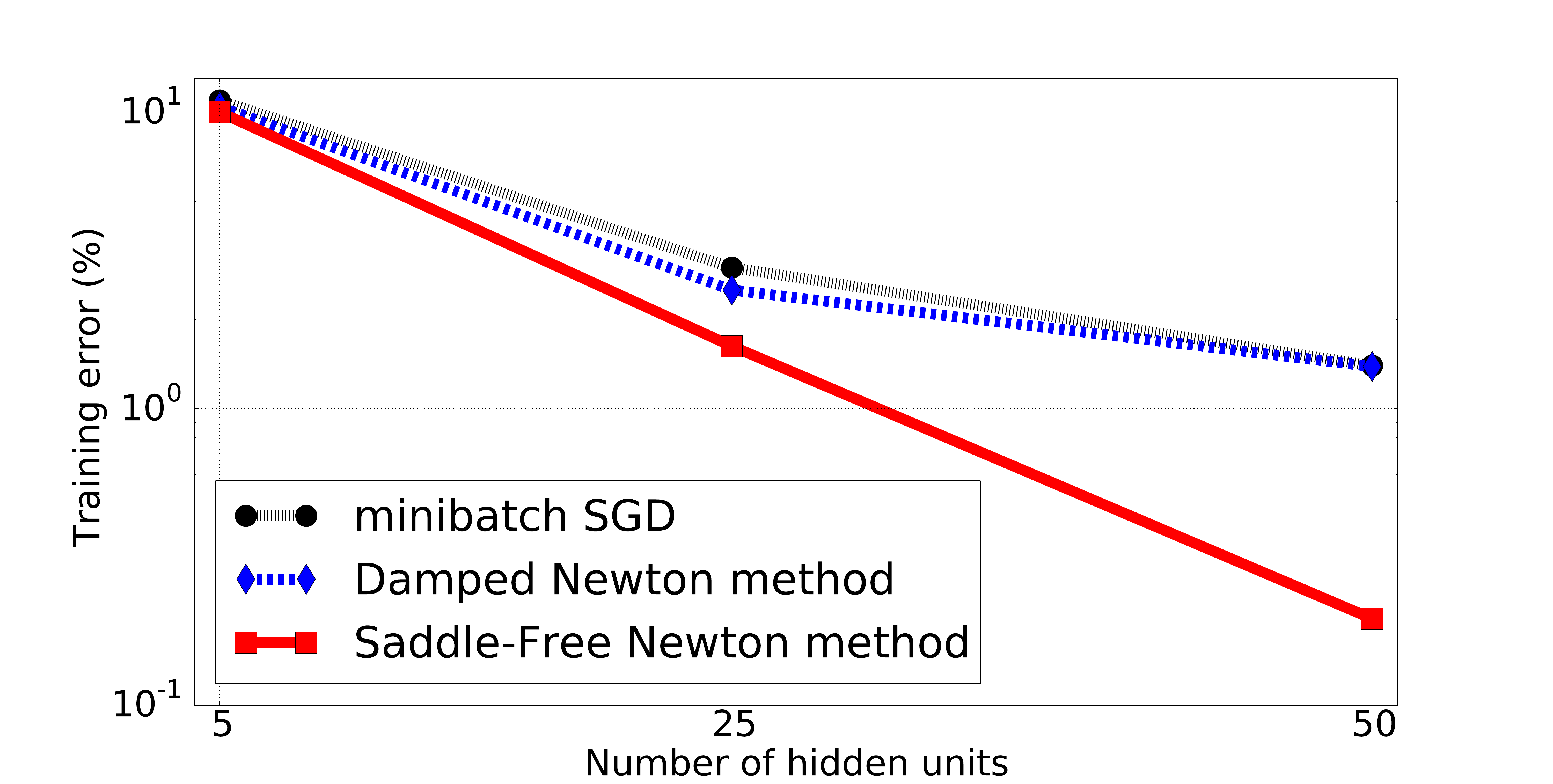}
    \end{minipage}
    \hfill
    \begin{minipage}{0.49\textwidth}
        \centering
        \includegraphics[width=1.\columnwidth, clip=true, trim=3cm 0cm 6cm 0cm]{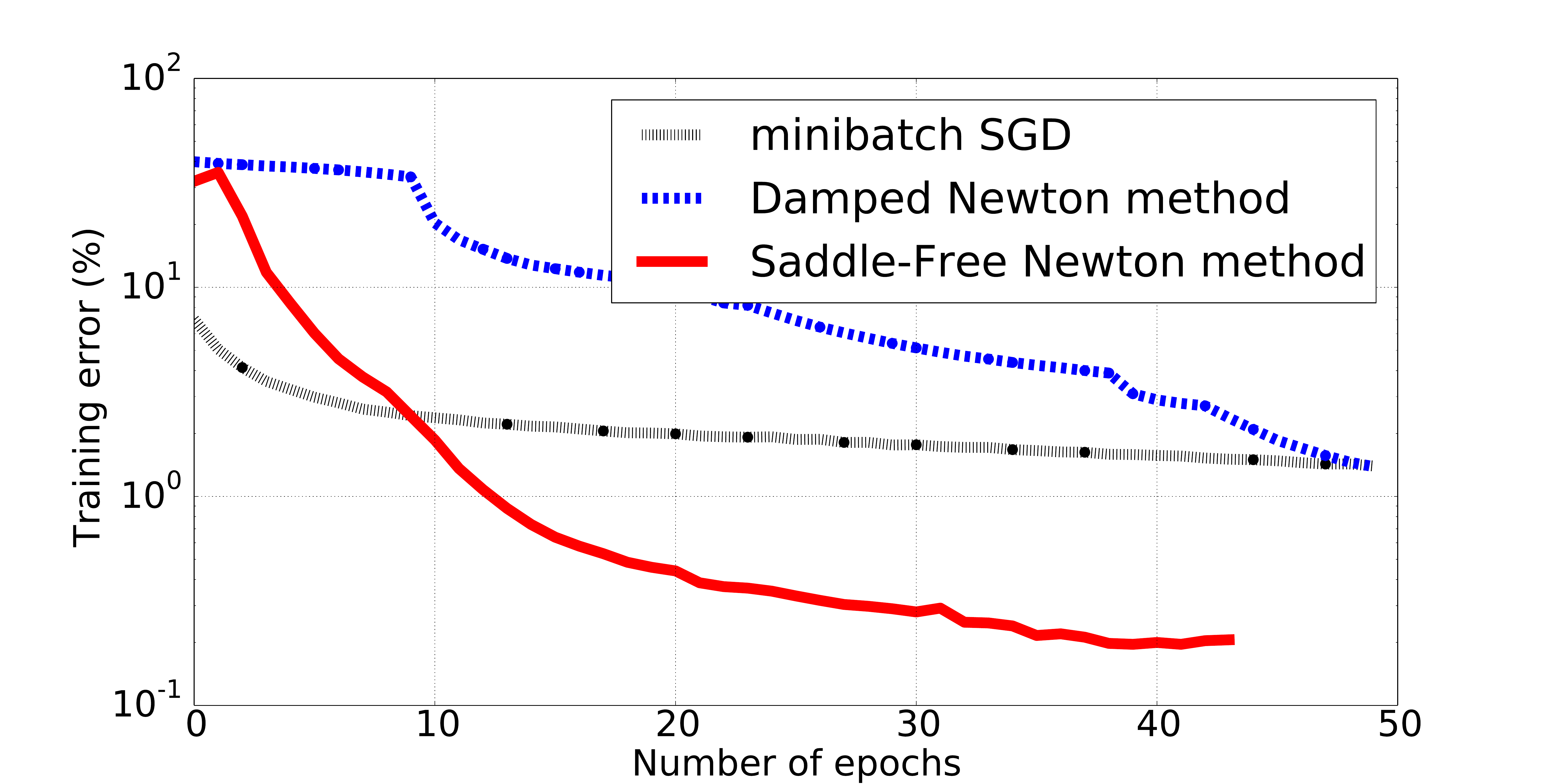}
    \end{minipage}
    \\
    \begin{minipage}{0.49\textwidth}
        \centering
        (a) 
    \end{minipage}
    \hfill
    \begin{minipage}{0.49\textwidth}
        \centering
        (b) 
    \end{minipage}
    \\
    \begin{minipage}{0.49\textwidth}
        \centering
        \includegraphics[width=1.\columnwidth, clip=true, trim=3cm 0cm 6cm 0cm]{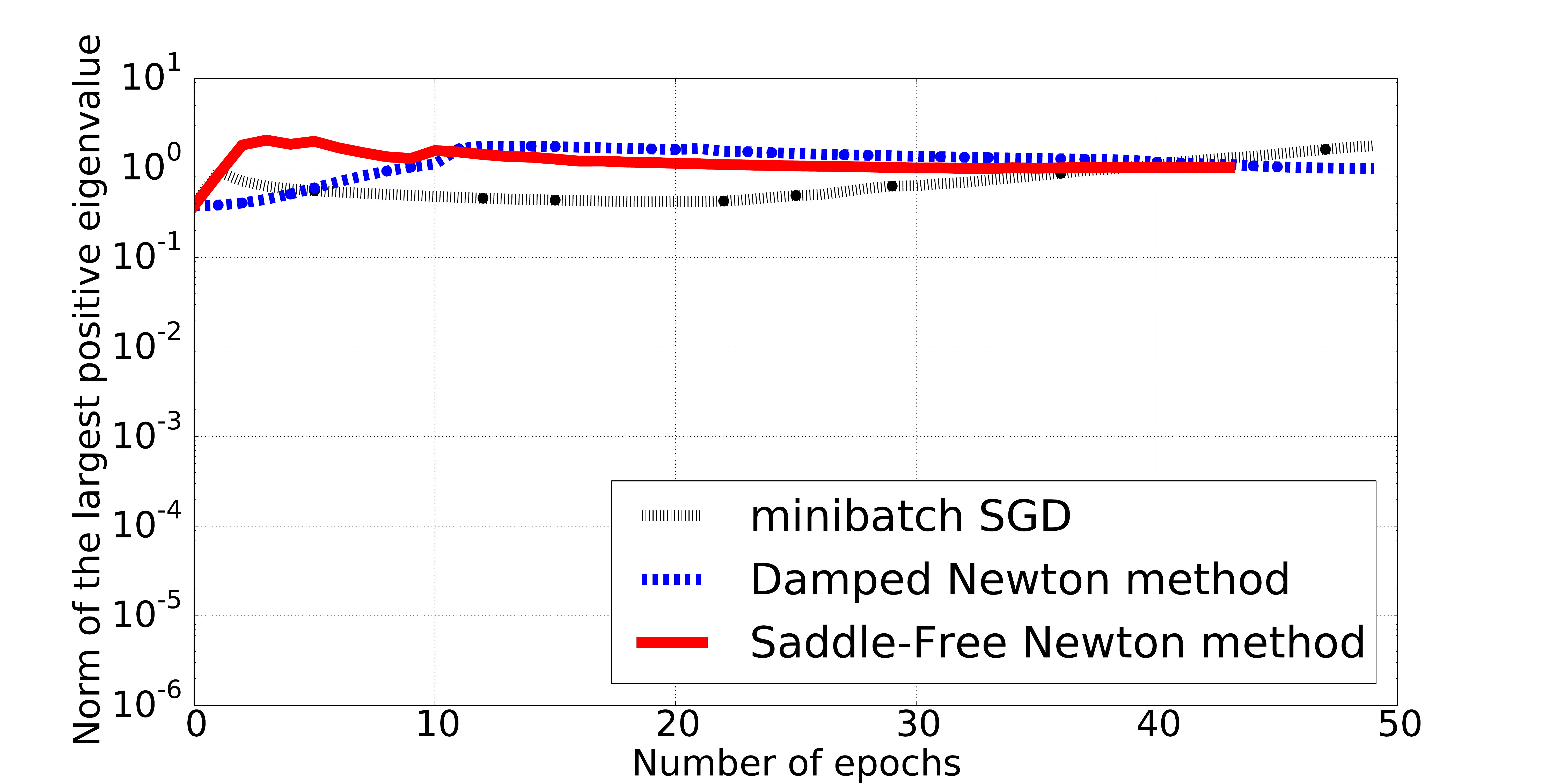}
    \end{minipage}
    \hfill
    \begin{minipage}{0.49\textwidth}
        \centering
        \includegraphics[width=1.\columnwidth, clip=true, trim=3cm 0cm 6cm 0cm]{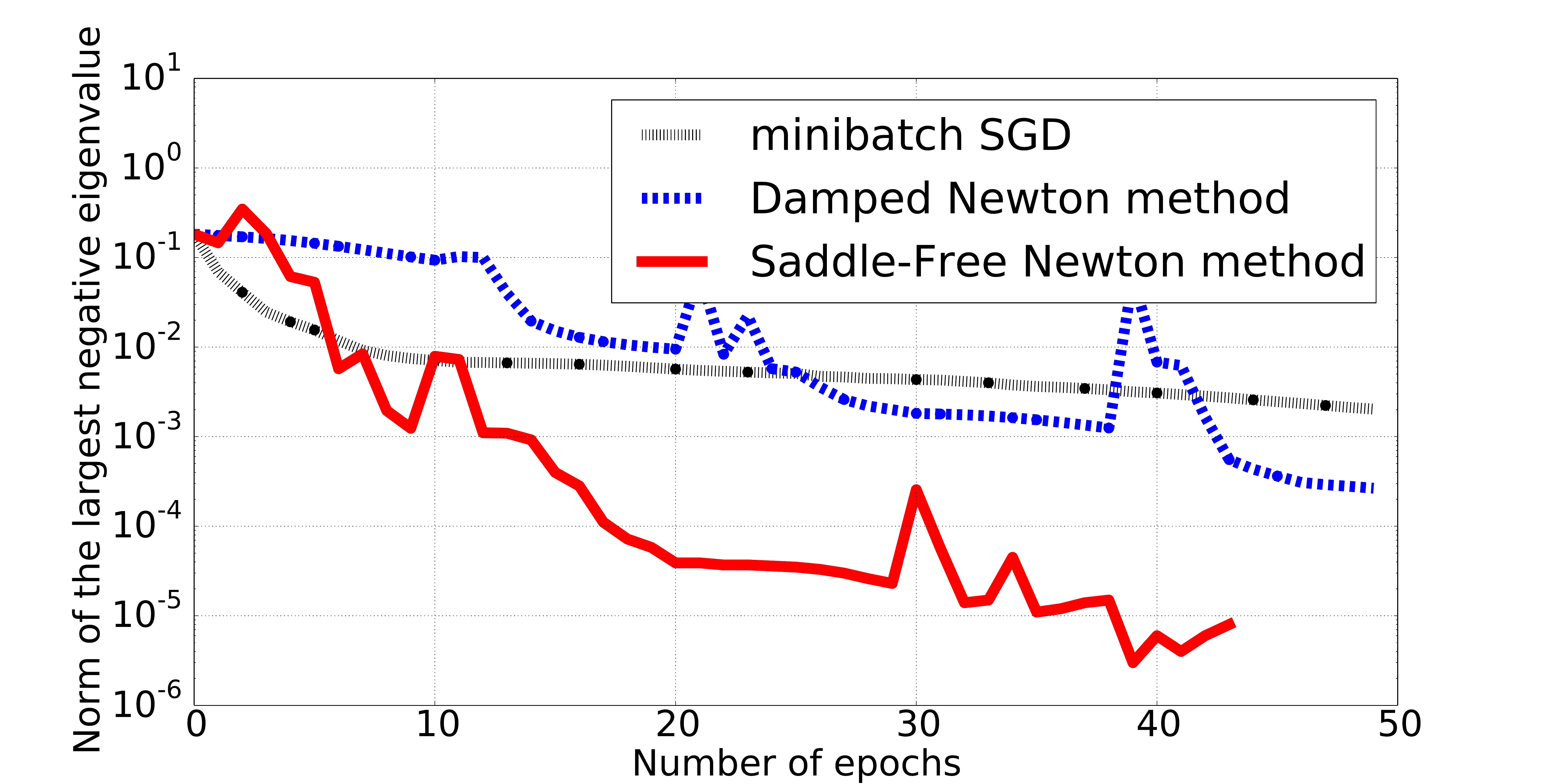}
    \end{minipage}
    \\
    \begin{minipage}{0.49\textwidth}
        \centering
        (c) 
    \end{minipage}
    \hfill
    \begin{minipage}{0.49\textwidth}
        \centering
        (d) 
    \end{minipage}
\caption{
Empirical evaluation of different optimization algorithms for a single layer
MLP trained on the rescaled MNIST dataset. In (a) we look at the minimum error
obtained by the different algorithms considered as a function of the model
size. (b) shows the optimum training curve  for the three algorithms. The error
is plotted as a function of the number of epochs. (c) looks at the evolution of
the norm of the largest positive eigenvalue of the Hessian and (d) at the norm
of the largest negative eigenvalue of the Hessian. 
}
\label{fig:results}
\end{figure}

\begin{table}[ht]
    \centering
    \begin{tabular}{c | c c c}
        \hline
         Model size & learning rate & momentum constant & minibatch size \\
        \hline
        \hline
        5 units & 0.074 & 0.031 & 10 \\
        25 units & 0.040 & 0.017 & 10 \\
        50 units & 0.015 & 0.254 & 1 \\ 
        \hline
    \end{tabular}
    \caption{Best performing hyper-parameters for stochastic gradient descent.}
    \label{table:hyper_sgd}
\end{table}

The results of these experiments are visualized in Figure~\ref{fig:results}.
Figure~\ref{fig:results}~(a) looks at the minimum error reached by different
algorithms as a function of the model size. The plot provides empirical
evidence that, as the dimensionality of the problem increases, the number of
saddle points also increases (exponentially so). We argue that for the larger
model (50 hidden units), the likelihood of an algorithms such as SGD or Newton
method to stop near  a saddle point becomes higher (as the number of saddle
points is much larger) and therefore we should see these algorithms perform
worse in this regime. The plot confirms our intuition. We see that for the 50
hidden units case, saddle-free outperforms the other two methods considerably. 

Figure~\ref{fig:results}~(b) depicts the training error versus the number of
epochs that the model already went through. This plot suggest that saddle-free
behaves well not only near a critical point but also far from them, taking a
reasonable large steps. 

In Figure~\ref{fig:results}~(c) we look at the norm of the largest positive
eigenvalue of the Hessian as a function of the number of training epochs for
different optimization algorithms. Figure~\ref{fig:results}~(d) looks in a
similar fashion at the largest negative eigenvalues of the Hessian. Both these
quantities are approximated using the Power method. The plot clearly shows that
initially there is a direction of negative curvature (and therefore we are
bound to go near saddle points). The norm of the largest negative eigenvalue is
close to that of the largest positive eigenvalue.  As learning progresses, the
norm of the negative eigenvalue decreases as it is predicted by the theory of 
random matrices \cite{Bray07} (think of the semi-circular distribution being shifted 
to the right).  For stochastic gradient descent and Damped Newton method, however,
even at convergence we still have a reasonably  large negative eigenvalue,
suggesting that we have actually ``converged'' to a saddle point rather than a
local minimum. For saddle-free Newton method the value of the most negative
eigenvalue decreases considerably, suggesting that we are more likely to have
converged to an actual local minimum.

\begin{figure} 
\centering
    \includegraphics[width=.49\textwidth,clip=true, trim=4cm 0cm 4cm 0cm]{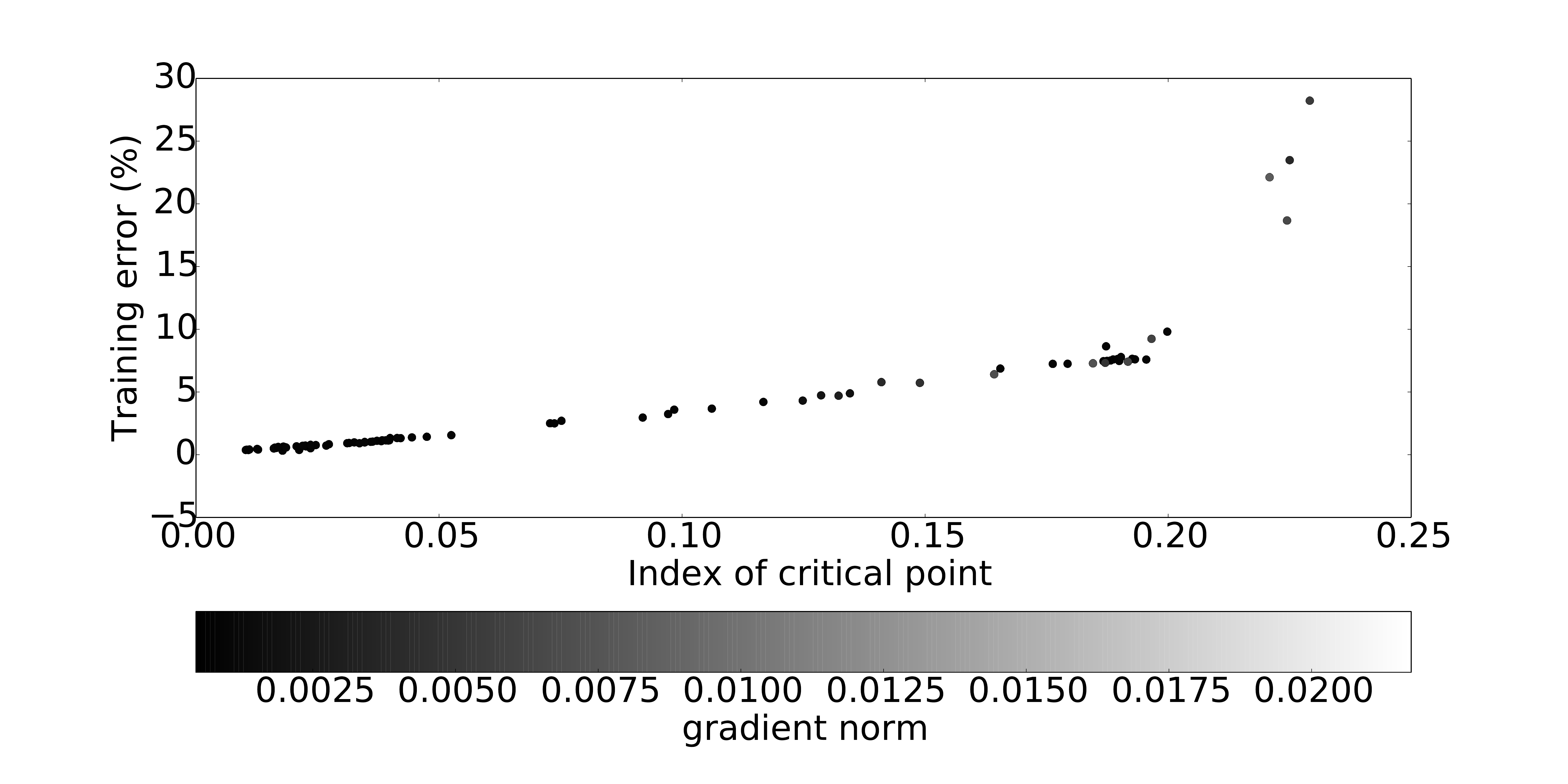}
    \includegraphics[width=.49\textwidth,clip=true, trim=3cm 0cm 4cm 0cm]{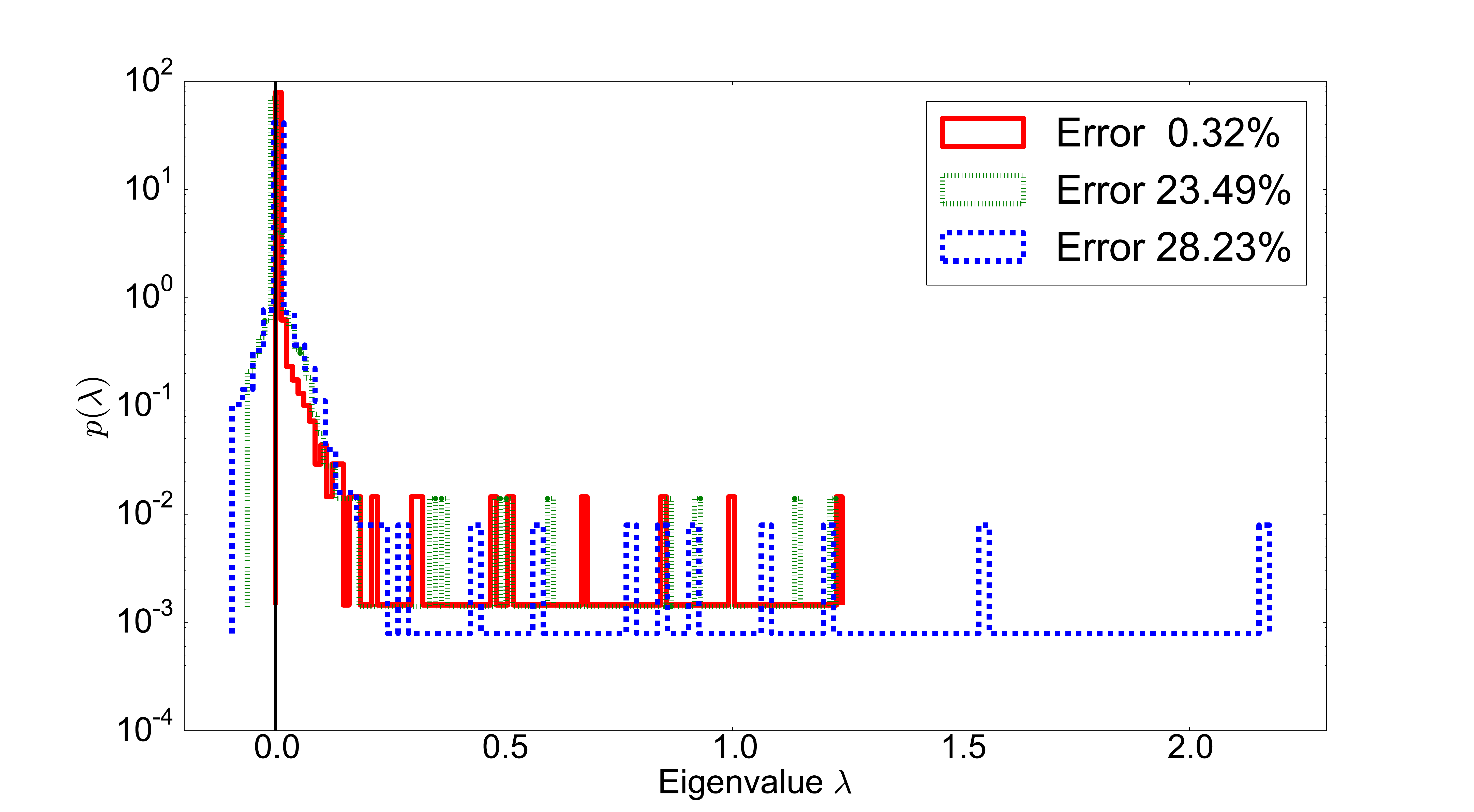}
    \caption[Training error versus the index and distribution of eigenvalues for different critical points]{
        The plot on the left depict the training error versus the index
        (fraction of negative eigenvalues) for different critical points found
        nearby the path taken by different runs of the saddle-free Newton
        method. The critical points are discovered using the Newton method.
        Note that the gray level of each point is given by the norm of the
        gradient where the Hessian was measured.  It can be regarded as a
        measure of noise (how far from the actual critical point we have
        actually converged). The plot on the right shows the distribution of
        eigenvalues of the Hessian for three different critical points selected
        based on their error. Note that the y-axis is on a log scale. 
    } 
    \label{fig:index} 
\end{figure}

Figure~\ref{fig:index} is an empirical evaluation of whether the properties
predicted by \citet{Bray07} for random Gaussian error functions hold for neural
networks. 

To obtain this plot we used the Newton method to discover nearby critical
points along the path taken by the saddle-free Newton algorithm.  We consider
20 different runs of the saddle-free algorithm, each using a different random
seed. We then run 200 jobs. The first 100 jobs are looking for critical points
near the value of the parameters obtained after some random number of epochs
(between 0 and 20) of a randomly selected run (among the 20 different runs) of
saddle-free Newton method.  To this starting position uniform noise is added of
small amplitude (the amplitude is randomly picked between the different values
$\{10^{-1}, 10^{-2}, 10^{-3}, 10^{-4}\}$ The last 100 jobs look for critical points 
near  uniformally sampled
weights (the range of the weights is given by the unit cube).  
The task (dataset and model) is the same as the one used previously.

In Figure~\ref{fig:index}, the plot on the left shows the index of the critical
point (fraction of negative eigenvalues of the Hessian at the critical point)
versus the training error that it obtains.  This plot shows that all critical
points, projected in this plane, align in a monotonically decreasing curve, as
predicted by the theoretical results on random Gaussian error
functions\citep{Bray07}. This provides evidence that most critical points with 
corresponding to large error has to be, with high probability, a saddle point, 
and not a local minima. 

The plot on the right looks at the distribution of the eigenvalues of the
Hessian at 3 different critical points picked according to the error that they
realise. Note that the plot is on a log scale on the y-axis. These
distributions \emph{do not follow} the semi-circular rule, as predicted by the
theory of random matrices. This is probably caused by the structure of the
neural network (and of the task). However, the generic observation of
\citep{Bray07}, that as the error decreases the distribution shifts to the
right seems to hold, with the exception of the peak that we have around $0$.
The fact that we have a large number of eigenvalues at $0$, and a few eigenvalues
that are sufficiently large suggests that any of these saddle-points are
surrounded by plateaus, in which the different algorithms might end up taking a
suboptimal step.

\section{Conclusion}

In the introduction of this work we provided a thorough literature review of
works that argue for the prevalence of saddle points in large scale
non-convex problems or how learning addresses negative curvature.  We tried to
expand this collection of results by providing an intuitive analysis of how
different existing optimization techniques behave near such structures. Our
analysis clearly shows that while some algorithms might not be ``stuck'' in the
plateau surrounding the saddle point they do take a suboptimal step. 

The analysis also suggests what would be the optimal step. We extend this
observation, that was done prior to this work, to a proper optimization
algorithm by relying on the framework of generalized trust region methods.
Within this framework, at each step, we optimize a first order Taylor
approximation of our function, constraint to a region within which this
approximation is reliable. The size of this region (in each direction) is
determined by how different the first order approximation of the function is
from the second order approximation of the function. 

From this we derive an algorithm that we call saddle-free Newton method, that
looks similarly to the Newton method, just that the matrix that we need to
invert is obtained from the Hessian matrix by taking the absolute value of all
eigenvalues. We show empirically that our claims hold on a small model trained 
on a scaled-down version of MNIST, where images are scaled 
to be $10 \times 10$ pixels.

As future work we are interested in mainly two directions. The first direction 
is to provide a pipeline for saddle-free Newton method that allows to scale 
the algorithm to high dimensional problems, where we can not afford to compute 
the entire Hessian matrix. The second direction is to further extend the 
theoretical analysis of critical points by specifically looking at neural 
network models. 

\subsubsection*{Acknowledgments}
The authors would like to acknowledge the support of the following agencies for
research funding and computing support: NSERC, Calcul Qu\'{e}bec, Compute Canada,
the Canada Research Chairs and CIFAR. We
would also like to thank the developers of
Theano~\citep{bergstra+al:2010-scipy,Bastien-Theano-2012}.
Razvan Pascanu is supported by a Google DeepMind Fellowship.

\bibliography{myrefs}
\bibliographystyle{natbib}
\end{document}